\documentclass[letterpaper, 10 pt, conference]{ieeeconf}  

\IEEEoverridecommandlockouts                              

\pdfoutput=1

\usepackage{amsmath} 
\usepackage{subfig}
\usepackage{tabularx}
\usepackage{graphicx}
\usepackage{caption}

\newtheorem{defn}{Definition}[section]
\newtheorem{theorem}{Theorem}
\newtheorem{lemma}[theorem]{Lemma}
\newtheorem{prop}{Proposition}

\usepackage{algorithm}
\usepackage{algpseudocode}

\title{\LARGE \bf
Partitioning Strategies and Task Allocation for Target-tracking with Multiple Guards in Polygonal Environments
}

\author{Hamid Emadi$^{1}$ , Tianshuang Gao$^{2}$ and Sourabh Bhattacharya$^{3}$
\thanks{$^{1}$Hamid Emadi and Sourabh Bhattacharya are with the Department of Mechanical Engineering, Tianshuang Gao is with the Department of Computer Science,
		Iowa State University, Ames, IA 50011
        {\tt\small \{emadi,sbhattac,tsgao\}@iastate.edu}}%
        }

\begin{document}

\maketitle
\thispagestyle{empty}
\pagestyle{empty}


\begin{abstract}
This paper presents an algorithm to deploy a team of {\it free} guards equipped with omni-directional cameras for tracking a bounded speed intruder inside a simply-connected polygonal environment. The proposed algorithm partitions the environment into smaller polygons, and assigns a guard to each partition so that the intruder is visible to at least one guard at all times. Based on the concept of {\it dynamic zones} introduced in this paper, we propose event-triggered strategies for the guards to track the intruder. We show that the number of guards deployed by the algorithm for tracking is strictly less than $\lfloor {\frac{n}{3}} \rfloor$ which is sufficient and sometimes necessary for coverage. We derive an upper bound on the speed of the mobile guard required for successful tracking which depends on the intruder's speed, the road map of the mobile guards, and geometry of the environment. Finally, we extend the aforementioned analysis to orthogonal polygons, and show that the upper bound on the number of guards deployed for tracking is strictly less than $\lfloor {\frac{n}{4}} \rfloor$ which is sufficient and sometimes necessary for the coverage problem. 

\end{abstract}

\section{INTRODUCTION}

In the past decade, mobile robotic networks have become a ubiquitous part of human society \cite{viglietta2012guarding,ding2012coordinated,quintero2010optimal,masehian2014sensor}. An important application has been in the area of surveillance and data gathering. Although, electronic and biometric techniques are emerging rapidly in security applications, vision-based monitoring using static camera networks is still the most prevalent technique used for persistent surveillance \cite{roy2012camera}. However, the scalability of this technique is poor due to the huge amount of data acquired by modern camera networks. Mobile camera networks alleviate the problem of {\it data deluge} to a significant extent. In this work, we explore a scenario in which a team of mobile agents that can visually track entities in the environment are deployed in a surreptitious manner for tracking a mobile intruder.



Target tracking refers to the problem of tracking a mobile object, called a {\it target}.
Based on the sensing modality and sensing constraints, there is a range of problems
that can be addressed under this category. Several variants of the target-tracking problem have been considered in the past that consider constraints in motion as well as sensing constraints for both agents. For an extensive discussion regarding the previous work on target tracking and its applications, we refer the reader to \cite{murrieta2007surveillance, bhattacharya2011cell} .

In an adversarial setting, the target-tracking problem gives rise to a visibility-based pursuit evasion game \cite{guibas1997visibility}. Tools from differential game theory have been used to investigate motion and control strategies for the observer. 
However, the aforementioned tools fail to elucidate the dependence of the tracking performance on the geometric complexity of the environment. In \cite{guillermo2016}, the authors capture this relationship by drawing an explicit connection between target-tracking and art-gallery problems involving mobile guards \cite{o1987art}. This paper is in a similar vein, and investigates partitioning techniques for a polygon to deploy guards for tracking a mobile intruder.


The art gallery problem is a classic problem in which the goal is to deploy the minimum number of guards in a polygonal environment to ensure visual coverage of the entire environment. Generally, the guards in the art gallery problem are assumed to be \textit{point guards} which are stationary agents equipped with omni-directional camera having infinite sensing range. Although, the art gallery problem for several classes of polygons is NP-hard~\cite{lee1986computational}, there are tight bounds on the number of stationary guards that can cover a simple $n$-sided polygon and non-simple polygons containing holes~\cite{chvatal1975combinatorial,o1987art,hoffmann1991art}. The notion of mobile guards was introduced in \cite{avis1981optimal}. Mobile guards are classified as edge, diagonal or free guards, depending on the path which they are patrolling. Similar to the point guards, there are tight upper and lower bounds on the number of diagonal and edge guards required to cover a polygon\cite{o1983galleries,o1987art}.  

Beside art-gallery problems, there have been several works that have addressed the coverage problem in a polygon. In \cite{tokekar2014polygon}, the authors study the problem of guarding a polygon under the $\Delta$-guarding constraint. This constraint ensures that all sides of a convex object are visible in the environment at all times. The problem of inspecting an entire polygon with a group of mobile guards is called the \textit{watchman's route} problem. In \cite{katz2011guarding,biedl2016guarding}, two variations of the watchman's route problem, one in which the guards travel on minimum-length route and the other in which the minimum number of mobile guards are deployed, are examined. These algorithms provide a practical method to search an intruder inside a polygon. Authors in \cite{ghodsi2016clearing} propose a motion-planning algorithm for a group of sliding robots, assuming that they move along the pre-located line segments with a constant speed to detect all the evaders with unbounded speed. Such problems are classified as \textit{search} problems~\cite{viglietta2012guarding}. A visibility based search in a polygon wherein all guards and intruders have bounded speeds is examined in \cite{tovar2008visibility}. In \cite{obermeyer2011complete}, the authors present a rotation schedule for a group of static searchlight sensors in order to detect all targets in the polygon. Furthermore, distributed algorithms are proposed in \cite{ganguli2006distributed,obermeyer2011multi} for a group of mobile guards to reach points from which the entire polygon is visible. Using these algorithms, a group of mobile sensors can fully cover the entire polygon from an arbitrary initial position. For an elaborate survey on art gallery problems and coverage algorithms for polygons, we refer the reader  to \cite{ghosh2007visibility,o1987art}. 

Another variant of the surveillance problem involves covering the entire environment with pan, tilt, zoom (PTZ) cameras for tracking or detecting an intruder. The challenge here is to plan and coordinate the motion of cameras to track all the evaders over long time intervals. Authors in \cite{quintero2010optimal,roy2012camera} provide a dynamic control over a network of PTZ camera in order to track an evader with a random trajectory. Some of the tracking algorithms are based on the current position of the target. For instance, a collision free motion strategy is proposed in~\cite{masehian2014sensor} for a differential drive mobile robot for tracking a target. On the other hand,~\cite{lavalle1997motion,ding2012coordinated} consider a scenario in which the sensor has some apriori knowledge about the target's behavior, and predicts its new position for tracking tasks. 

In this work, we formulate the tracking problem as a task allocation problem. The guards are assigned the task of tracking the intruder in a partition of the polygon. The contributions of this paper are as follows: (i) We propose a deployment algorithm to track the intruder which requires less than $\lfloor\frac{n}{3}\rfloor$ guards; the sufficient and sometimes necessary number of guards for coverage. This provides a new upper bound on the number of guards required for target tracking. (ii) We introduce the idea of`dynamic zones", and use it to propose event-triggered strategies for the guard to track the intruder. (iii) We provide an upper bound on the speed of the guard required to track the intruder. This bound encompasses the geometric parameters of the environment. (iv) We extend the techniques to orthogonal polygons to propose deployment strategy for $\lfloor\frac{n}{4}\rfloor$ guards to track successfully.



The rest of the paper is organized as follows. In section II, we present the problem statement. In section III, we propose a deployment of guards in a general polygon by partitioning it into basic polygons. In section IV, the deployment and strategy of guards in the basic polygons are described. In section V, we focus on a special class of polygons, which is called orthogonal polygons. In Section VI, we conclude the paper with a few remarks.


\section{Problem Statement}
Consider a team of mobile agents, called {\it guards}, inside a simply-connected polygonal environment $P$ with $n$-sides. Assume that each observer is equipped with an omni-directional camera having an infinite sensing range. The environment contains another mobile agent called the {\it intruder}. The objective of the guards is to track the intruder inside the polygon at all times. The speed of the guard and intruder are denoted as $v_e \in [0,\bar{v}_e]$ and $v_p \in [0,\bar{v}_p]$, respectively.  Two points inside the polygon are mutually \textit{visible} if the line segment joining them, called the {\it line of sight} (LOS), is contained in $P$. Since we consider infinite sensing range, LOS can be obstructed only by obstacles (i.e. the reflex corners of the environment).

According to the art gallery theorems, $\lfloor \frac{n}{3}\rfloor$ static guards are sometimes necessary and always sufficient to cover the interior of a simply connected polygon $P$ of $n$-vertices. In a coverage problem, the objective is to cover the entire environment simultaneously, regardless of the intruder's position. There are some scenarios in which certain intruders in the environment should be tracked. It would be extravagant for these scenarios to cover the entire environment. Contrary to the coverage problem, it is sufficient to keep the intruder in the guards' field of view in the tracking problem. Since the intruder has bounded speed, deployment of the mobile guards with bounded speed is practical. In this work, we deploy a group of mobile and static guards to track the intruder within a polygon $P$, such that the total number of required guards is strictly less than $\lfloor \frac{n}{3}\rfloor$. The reduction of the number of guards comes from the motion ability of the guards. 
In \cite{guillermo2016}, the authors presented a deployment for diagonal guards based on the triangulation of the polygon. In this work, we investigate deployment strategies for {\it free guards}, i.e., guards that are free to move inside the polygon.

\section{Target Tracking in General Polygons }
In this section, we present a technique to partition a general polygon into smaller polygonal regions so that guards can be allocated to each region for tracking the intruder. This can be modeled as a multi-robot task allocation problem (MTAP) wherein guarding the different regions are tasks that need to be allocated to the robots. This can also be modeled as a resource allocation problem wherein guards are resources that need to be allocated to the different regions. The goal of task allocation is to assign robots to subtasks in order  to reach the performance of the system. The robots share their individual expected task to maximize the team performance. Algorithms to solve the matching problem for weighted bipartite multi-graphs are used to solve MTAP \cite{nam2014assignment}. There is a mathematical modeling and analysis of the collective behavior of dynamic task allocation in \cite{lerman2006analysis}. In this work, we assign a guard to a partition and the motion strategy of the mobile guard is given based on the dynamics of the intruder.

The following lemma presents the main idea of the partitioning technique:

\begin{lemma}\label{partitiong-general-polygon}
\cite{o1987art} Consider a polygon $P$ with $n \geq 10$ vertices, and $T$ a triangulation graph of $P$. There exists a diagonal $d$ in $T$ that divides $T$ into two partitions, one of which contains $k = 5, 6, 7,$ or $8$ arcs corresponding to edges of $P$.
\end{lemma}
From lemma~\ref{partitiong-general-polygon}, we can recursively partition a polygon $P$ into smaller polygons such that each partition contains $6,7,8,$ or $9$ edges. The process terminates with a polygon with less than 6 edges.
\begin{defn}
\textit{Minimal partitioning} is partitioning a polygon $P$ using diagonals of triangulation $T$ into smaller partitions such that a pair of partitions share at most one edge and except one partition, all partitions contain $6,7,8,$ or $9$ edges. The remaining partition may be a pentagon, quadrilateral or triangle.
\end{defn}

First, we show that the sufficient number of mobile guards to track a target is, in general, less than the number of static guards required to cover the entire polygon.
\begin{prop} \label{prop mobile guards}
Every polygon $P$, can be guarded by less than $\lfloor \dfrac{n}{3}\rfloor$ mobile guards if following statements hold:

\begin{enumerate}
\item Minimal partitioning $P$ contains at least $3$ partitions.
\item For every partition $P_i$ the sufficient number of mobile guards to track an intruder is less than $\lfloor \dfrac{n_i}{3}\rfloor$.
\end{enumerate} 
\end{prop}

\begin{proof}
Consider a triangulation $T$ of polygon $P$. According to lemma~\ref{partitiong-general-polygon}, any general $n$-sided polygon $P$ can be partitioned into $P_i$'s, such that each partition is a hexagon, septagon, octagon or nonagon. Furthermore, it may contain only one partition $P^\prime$ which is a triangle, quadrilateral or pentagon. Let $r$ denotes the number of $P_i$'s. Let $\hat{k}_i +k_i$ be the number of edges of partitions, where  $\hat{k}_i \in \{1,2\},$ and $k_i \in \{4,5,6,7\}$. Partitioning $P$ may result in one more partition $P^\prime$ with $k^\prime $ edges, where $k^\prime \in \{0,3,4,5\}$. Since the partitions are obtained from the original polygon, the following relation holds between $n, k_i,\hat{k}_i$ and $k^\prime$,
\begin{equation}
n+2(r-1)-k^\prime = \sum_{i=1}^{r}{k_i+\hat{k}_i}.
\end{equation}

We have assumed that each $P_i$ can be guarded by at least one guard less than the required number of static guards, that is $\lfloor \frac{k_i+\hat{k}_i}{3} \rfloor -1$. We need to show that the total number of guards is less than $\lfloor \frac{n}{3}\rfloor$. For those partitions, in which $P^\prime$ exists (i.e. $k^\prime \neq 0$) one more guard is needed to cover $P^\prime$. Therefore, the total number of guards is given by the following:
\begin{eqnarray} \nonumber
&&\sum_{i=1}^{r} {\lfloor \dfrac{k_i+\hat{k}_i}{3} \rfloor} -r+1 \leq \lfloor \dfrac{\sum_{i=1}^{r}k_i+\hat{k}_i}{3} \rfloor - r+1\\ \nonumber
 &\leq& \lfloor \dfrac{n+2(r-1)-k^\prime+2}{3} \rfloor - r+1 \\ \nonumber
 &\leq& \dfrac{n-k^\prime +2(r-1)+2}{3} -r+1 = \dfrac{n-k^\prime -r+3}{3}\\ 
 &<& \lfloor \frac{n}{3}\rfloor \quad  \text{for} \quad r \geq 3.
\end{eqnarray}
For those polygons, in which $P^\prime$ does not exist (i.e. $k^\prime = 0$), total number of guards is given by the following:
\begin{eqnarray} \nonumber
&&\sum_{i=1}^{r} {\lfloor \dfrac{k_i+\hat{k}_i}{3} \rfloor} -r \leq \lfloor \dfrac{\sum_{i=1}^{r}k_i+\hat{k}_i}{3} \rfloor - r\\ \nonumber
 &\leq& \lfloor \dfrac{n+2(r-1)}{3} \rfloor - r \leq \dfrac{n+2(r-1)}{3} -r = \dfrac{n-r-2}{3}\\ 
 &<& \lfloor \frac{n}{3}\rfloor \quad  \text{for all}\quad r.
\end{eqnarray}
\end{proof}

\section{Tracking in Basic Polygons} 
In this section, we present deployment strategies for mobile guards within each partition.

According to proposition \ref{prop mobile guards}, as the number of partitions increases, the sufficient number of guards is decreased. 

We propose motion strategy, and derive the maximum speed required for the mobile guards to track an intruder within each $P_i$ such that the second condition in proposition \ref{prop mobile guards} is satisfied. Let $S_i^*$ denote the \textit{star region} which is defined as the bounded area, inside the environment, across the reflex vertex $O_i$. In other words, $S_i^*$ is the set of all points in $P$ such that the connecting line between those points and $O_i$ lies in $P$, and lies in the area constructed by two edges of reflex vertex $O_i$. Figure~\ref{star-region} shows the star region in a polygon. In the proposed technique, the motion strategy of the guards are based on the number of disjoint star regions in $P_i$. The following propositions characterizes the maximum number of disjoint star regions in hexagons, septagons and octagons. 

	   \begin{figure}
	   \centering
	   \includegraphics[scale=.3]{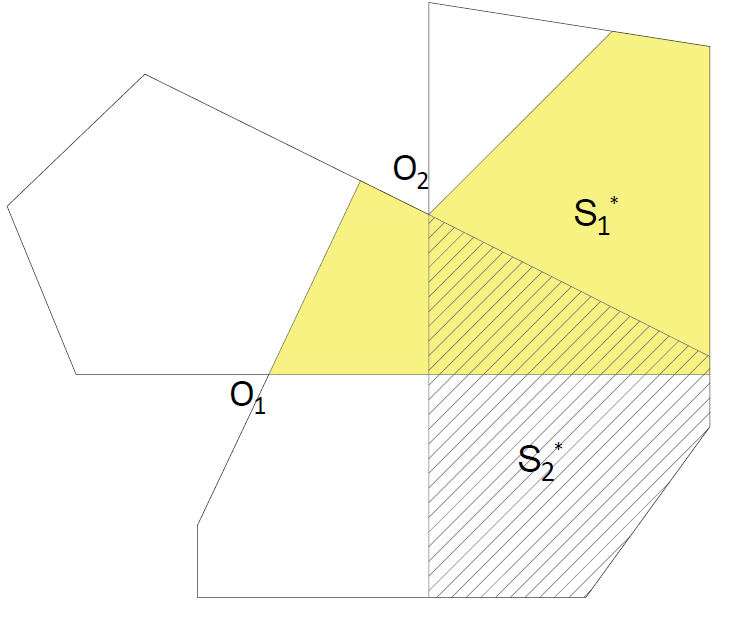}
	   \caption{Star regions in a polygon containing two reflex vertices. $S_1^*$ yellow region, and $S_2^*$ hatched region.}
	   \label{star-region}
	   \end{figure}	

\begin{lemma} \label{lemma1}
	Every polygon $P$ with a reflex vertex and $r\geq 2$ non-intersecting edges with the star region has at least $r+3$ edges.  Moreover, $P$ contains at least $5$ edges for $r=1$.
\end{lemma}
\begin{proof}
	Any reflex vertex contains two intersecting line segments. Extending these line segments separates the entire space into four quadrants, where the star region is the first quadrant, and numbering proceeds counterclockwise. All edges should lie either in quadrant one, two or four. For $r>1$, two edges are used to construct the reflex vertex, and consequently they have intersection with the star region, $r$ edges may be on quadrants two or four, but one more edge must intersect with the star region because of the simply connectedness of the polygon. Therefore, the polygon contains $r+3$ edges. For $r=1$, two edges are needed to connect the non-intersecting edge to the other edge of the reflex corner. As a consequence, $P$ contains at least $5$ edges.	
\end{proof}
\begin{lemma} \label{lemma-2}
	In a pentagon with one reflex vertex and two non-intersecting edges with the star region, sum of each pair of angles of pentagon on each side of star region must be less than $\pi$ (i.e. $\alpha_1+\alpha_2 < \pi$).
\end{lemma}
\begin{proof}
	Figure~\ref{star-region-proof} shows a pentagon with one reflex vertex and two edges that do not intersect with the star region ($AB$ and $CD$). Extending two edges corresponding to the reflex vertex, separates the entire space into four regions, one of which contains the star region. $OA$ and $OD$ are used to draw a reflex vertex, and other edges should be on second and fourth quadrants, while the first quadrant contains the star region. Therefore, without loss of generality, every pentagon containing one non-intersecting edge with star region, can be considered as Figure~\ref{star-region-proof}. From Figure~\ref{star-region-proof},
	\begin{eqnarray} \nonumber
	\alpha_1+\alpha_2 + \beta_1+\beta_2= 2\pi,\\ \nonumber
	\alpha_1+\alpha_2 + \gamma_1+\gamma_2+\gamma_3+\gamma_3 =2\pi.
	\end{eqnarray} 
	Since $\gamma_1+\gamma_2+\gamma_3 =\pi$ and $\gamma_3 > 0$, $\alpha_1+\alpha_2 < \pi$. The same argument holds for $\alpha_1^\prime+\alpha_2^\prime < \pi$.
\end{proof}

\begin{figure}[!tbp]
  \centering
  \subfloat[]{\includegraphics[width=0.25\textwidth]{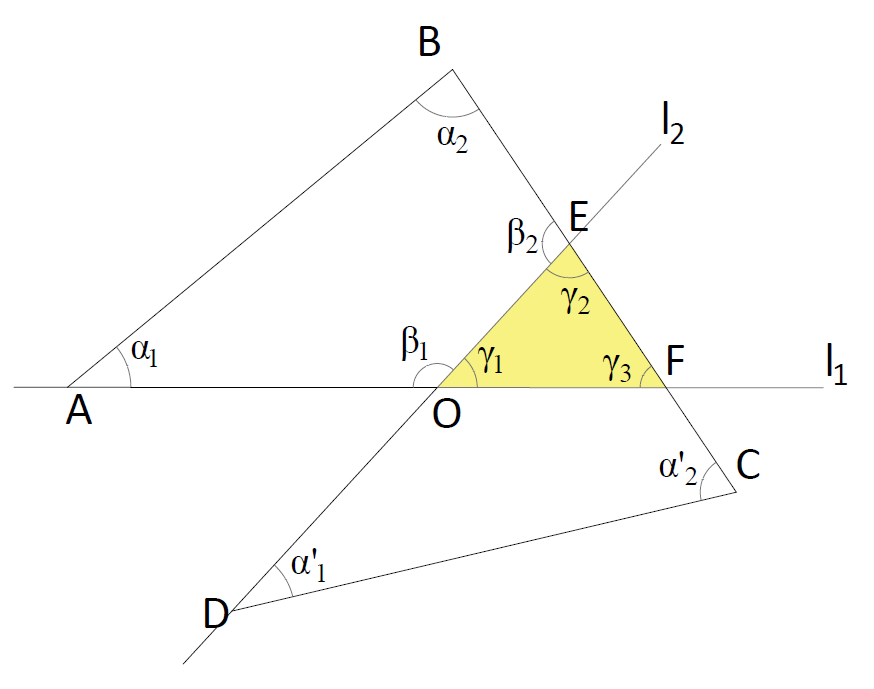}\label{star-region-proof}}
  \hfill
  \subfloat[]{\includegraphics[width=0.2\textwidth]{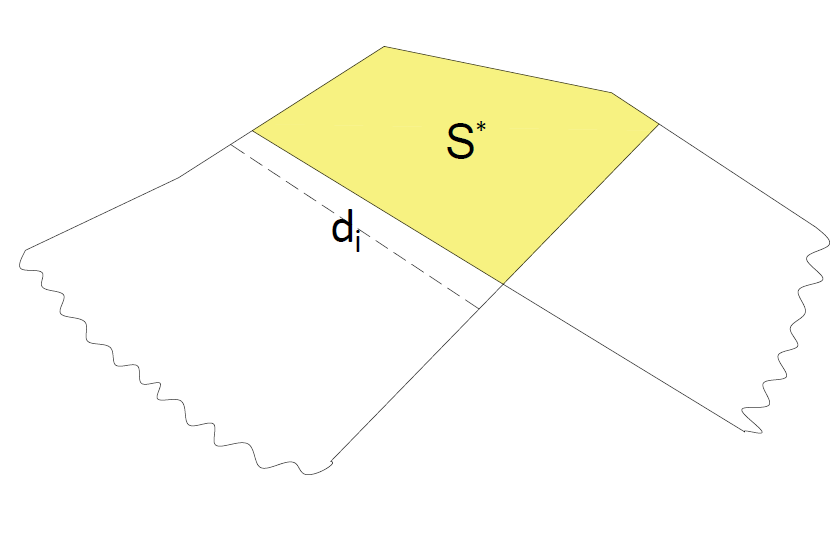}\label{star-region-proof2}}
  \caption{(a) A pentagon with one reflex vertex and two non- intersecting edges with the star region.(b)Line segment $d_i$ to partition a polygon $P_n$. }
\end{figure}

\begin{prop} \label{number of star regions}
	Every polygon containing $2,3,4$ or $5$ disjoint star regions must have at least $6,7,9$ or $10$ edges, respectively.
\end{prop}

\begin{proof}
	We consider a $n$-sided polygon $P$, and let $n_d \in \{2,3,4,5\}$ be the number of disjoint star regions. Each star region is a convex polygon with two edges inside $P$, and only one common vertex with reflex vertices of $P$ and some parts of the boundary of $P$ (yellow region in Figure~\ref{star-region-proof2}). Based on the definition of a star region, we can find line segments $d_1,\dots,d_{n_d-1}$ which partition $P$ into $n_d$ parts, and each part contains one star region. Figure~\ref{star-region-proof2} shows a plausible choice of $d_1$. Now, let $G(v,e)$ be a dual graph for $P$ such that each partition with only one star region is a node in $G$. An edge exists between a pair of nodes if their corresponding partitions are separated by a common $d_i$. Since $P$ is simply connected, graph $G$ is connected. Moreover $G$ has $n_d$ nodes and $n_d-1$ edges. Therefore, $G$ is a tree. All possible non-isomorphic trees with $n_d \in \{2,3,4,5\}$ nodes are illustrated in Figure~\ref{octagon-graph}. 
	
	Each node with degree one corresponds to a polygon with a reflex corner and a non-intersecting edge with star region. As a result of lemma~\ref{lemma1}, each node with degree one must be a polygon which contains at least $5$ edges. Additionally, nodes with degree $3$ and $4$ must be hexagon and septagon, respectively. Based on lemma~\ref{lemma-2}, two consecutive nodes with degree two cannot both be a pentagon, which implies one of the partitions must contain at least $6$ edges. In order to find the total number of edges of $P$, we have to exclude the common edges in partitions (i.e. $d_i$'s). Moreover the edges belonging to $P$ which intersect $d_i$'s should be counted only once. Computing the total number of edges for the graphs for $n_d =2,3,4$ and $5$, leads to $n=6,7,9$ and $10$, respectively. Therefore, a hexagon, a septagon and an octagon can at most contain $2, 3$ and $3$ disjoint star regions, respectively.
		   
	   \begin{figure}
	   \centering
	   \includegraphics[scale=.3]{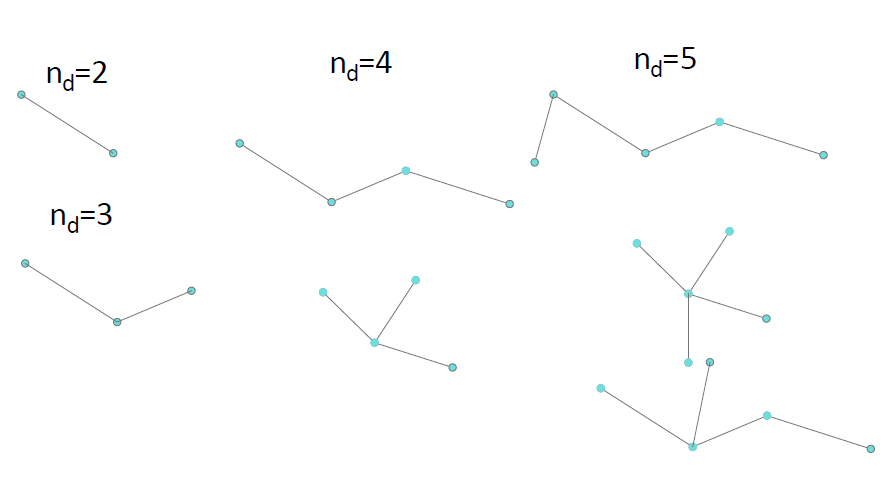}
	   \caption{Dual graph of $P$ for $n_d=2,3,4,5$. }
	   \label{octagon-graph}
	   \end{figure}		   

\end{proof}
\subsection{Tracking in a Hexagon}
In this section, we show that every hexagonal environment can be guarded by a mobile guard if the maximum speed of guard is greater than or equal to a threshold, which depends on the geometry of the environment and maximum speed of evader (i.e. $\bar{v}_p \geq v^*_p$ ). First, we propose an initial deployment of the pursuer for different initial positions of the evader based on the number of disjoint star regions in the environment, and then we present a strategy for the pursuer's motion. Based on the number of reflex vertices, every hexagon can be classified into the following categories:
\begin{itemize}
\item Case I: Convex hexagon (no reflex vertex).
\item Case II: Hexagon with one reflex vertex.
\item Case III: Hexagon with two reflex vertices.
\item Case IV: Hexagon with three reflex vertices.
\end{itemize}

In Case I, a static guard can cover the entire hexagon since it is convex. In case II, deploying only one static guard in $S^*$ can cover the entire hexagon. 

In case III, let $\gamma_p$ denotes the smallest path within the hexagon which connects $S^*_1$ and $S^*_2$. $S^*_1 \cap S^*_2 \neq \emptyset$ implies $|\gamma_p| =0$, and there exists at least one point in $S^*_1 \cap S^*_2$ which can cover the entire environment. Deploying one static guard in $S^*_1 \cap S^*_2$ would be enough to cover the hexagon. When $|\gamma_p| \neq 0$, the environment can be partitioned into three parts, which determines the strategy of mobile guard. This partitioning is based on defining a \textit{dynamic zone} denoted by $D_i$ around a reflex vertex $C_i$. 

Refer to Figure~\ref{dynamic-zone-1}. Assume the pursuer lies on point $H_0$. Let $l_i$ be the ray that begins from vertex $C_i$ and it lies in the area of visibility of $H_0$ (i.e. $l_i \in {\cal V} {(H_0)}$). For any line segment $d_i:=H_0H_1$ which connects $H_0$ to $S_i^*$, dynamic zone is defined as the set of all points in the environment whose distance from $l_1$ is less that $r_i=\frac{\bar{v}_e}{\bar{v}_p}d_i$. The green area in Figure~\ref{dynamic-zone-1} shows the dynamic zone. In general, $D_i$ encompasses a sector of a disc with radius of $r_i$, and the area between two parallel lines, one of which tangents to the disc and the other one passes through $C_i$. When the evader enters region $D_i$, the pursuer initiates motion from $H_0$ towards $H_1$ along the path $d_i$. The boundary of $D_i$ contains two main parts, $\partial D_i$ and $l_i$ belong to the free space.  $\partial D_i$ and $l_i$ correspond to the points $H_0$ and $H_1$, which implies that if the evader enters $D_i$ from $\partial D_i$ or $l_i$ then the pursuer should lie on $H_0$ and $H_1$, respectively. Based on the distance of the evader with respect to $\partial D_i$ and $l_i$, the pursuer should be on a certain position on $d_i$. In other words, there exists an isomorphic mapping from $D_i$ to $d_i$.
 
We can extend the concept of dynamic zone to the environments containing multiple reflex corners. First, we construct a road map which connects all star regions. Figure~\ref{dynamic-zone-3} shows a road maps which connects three star regions. Next, based on the road map we construct the dynamic zones around each corner. The dynamic zones for a certain road map should satisfy the following conditions:
\begin{figure}[!tbp]
  \centering
  \subfloat[]{\includegraphics[width=.25\textwidth]{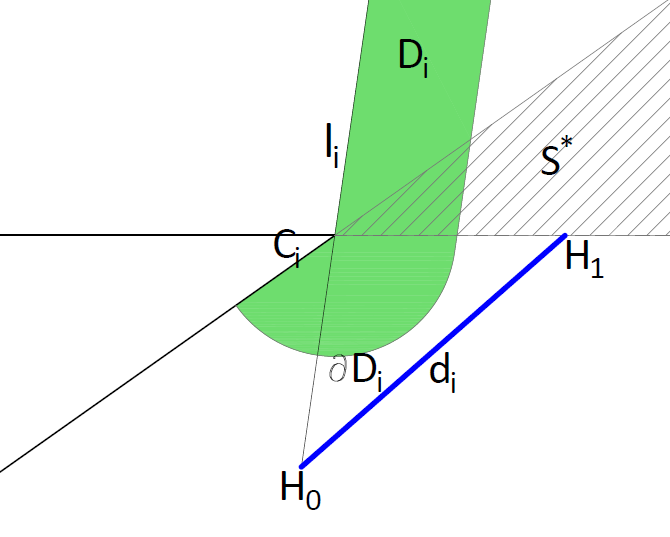}\label{dynamic-zone-1}}
  \hfill
  \subfloat[]{\includegraphics[width=.25\textwidth]{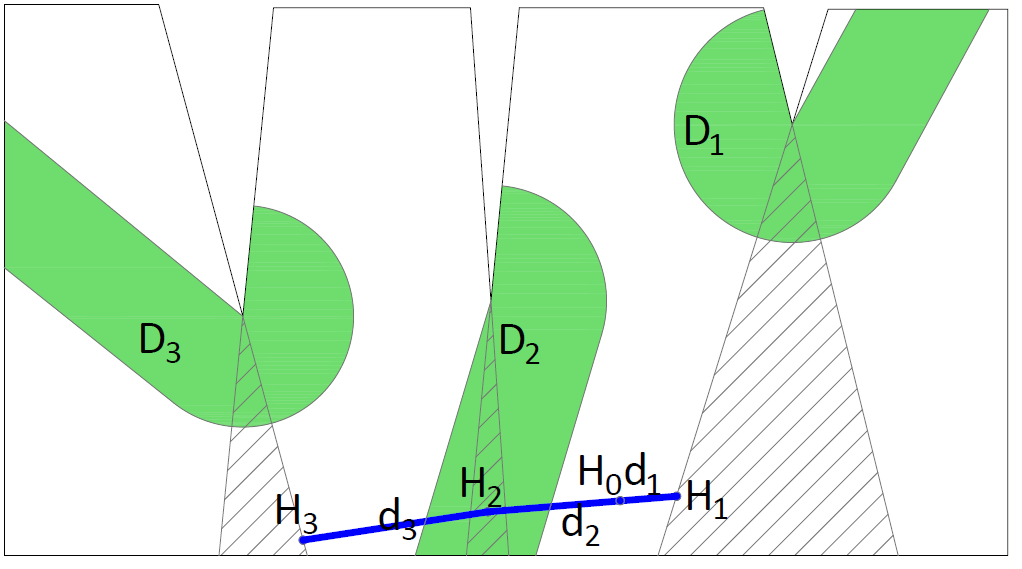}\label{dynamic-zone-3}}
  \caption{ (a) Dynamic zone around corner $C_i$,(b)Dynamic zones and road map among star regions}
\end{figure}
\begin{itemize}
\item Let $D_i$ and $D_j$ be two dynamic zones which are separated by a static zone. $D_i$ and $D_j$ should be consistent. By consistent we mean, two boundaries of $D_i$ and $D_j$ which are connected to the same static zone should correspond to the same point according to the pursuer's position.
\item Every two dynamic zones cannot contain common points except some points on their boundaries. 
\item Every $D_i$ cannot contain a reflex vertex except $C_i$.
\end{itemize}

\begin{figure}[!tbp]
  \centering
  \subfloat[]{\includegraphics[width=.25\textwidth]{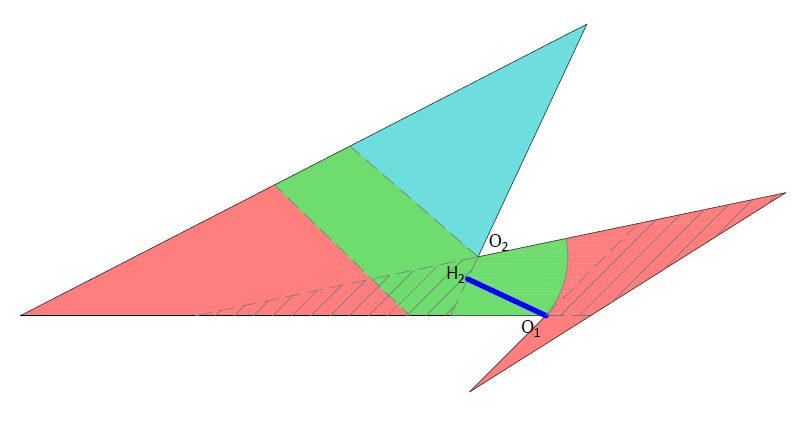}\label{hexagon_case_I}}
  \hfill
  \subfloat[]{\includegraphics[width=.2\textwidth]{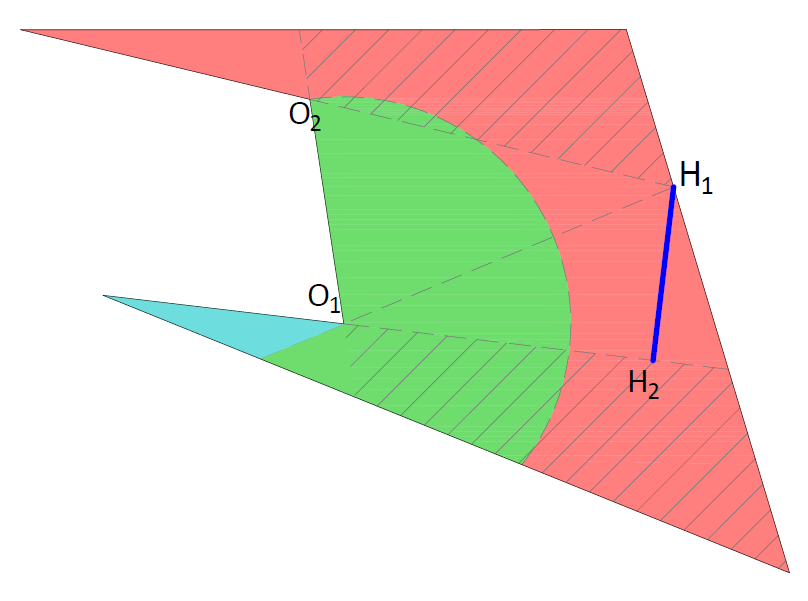}\label{hexagon_case_II}}
  \hfill
  \subfloat[]{\includegraphics[width=.3\textwidth]{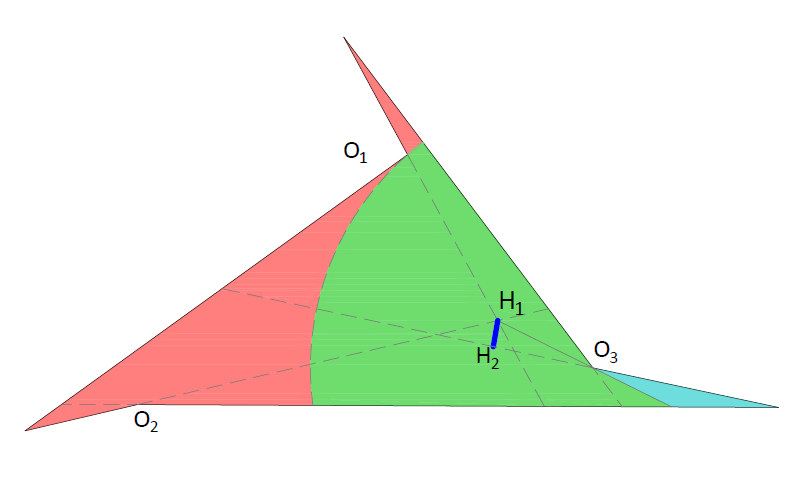}\label{hexagon_case_III}}
  \caption{Hexagonal environment (a),(b) Case III , c) Case IV. Evader's position in every colored partition, determines specific deployment and strategy of pursuer.}
\end{figure}

When a hexagon contains two disjoint star regions, the maximum radius of dynamic zone is the distance between two reflex vertices $o_{12}$. Therefore, the minimum $\bar{v}_p$ is $\dfrac{|\gamma_p|}{r_{12}}\bar{v}_e$.

In case IV, the environment has three reflex corners. If $S^*_1 \cap S^*_2 \cap S^*_3 \neq \emptyset$, the entire hexagon can be covered by only one static guard. For environments in which $S^*_1 \cap S^*_2 \cap S^*_3 = \emptyset$ (Figure~\ref{hexagon_case_III}), $\gamma_p$ is defined as the shortest path which connects the intersection set of two star regions. That is the shortest path between $S^*_i \cap S^*_j$ for $i \neq j, i,j=1,2 ,3$. Given $\gamma_p$, the deployment and motion strategy of the pursuer can be determined in the same way as described for case III. The following proposition summarizes the results in this subsection.

\begin{prop}\label{max-speed}
A mobile guard can track an evader in a hexagonal environment for infinite time if the following condition is satisfied:
\begin{eqnarray}
\bar{v}_p &\geq& v^*_p\\
v^*_p &=& \dfrac{|\gamma_p|}{r}\bar{v}_e,
\end{eqnarray}
where $\gamma_p$ is the shortest path which connects all star regions, and $r$ is the radius of arc segment in $D_i$.  
\end{prop}
\subsection{Tracking in a Septagon} 
In this subsection, we analyze the deployment and motion strategy for a single guard in a septagonal environment. Similar to the hexagonal environments, we can classify septagons according to the number of reflex corners, which is at most five.
\begin{itemize}
\item Case I: Convex septagons, which can be covered with one static guard.
\item Case II: Septagons containing only one reflex vertex, which can be covered with only one static guard which is placed in $S^*$.
\item Case III: Septagons containing two reflex vertices. When star regions are completely disjoint, a mobile guard with a maximum speed, given by proposition~\ref{max-speed} is needed. Otherwise one static guard is enough to cover the entire environment.
\item Case IV: Septagonal environments in which three reflex corners exist. If intersection of all star regions is not empty set (i.e. $S^*_1 \cap S^*_2 \cap S^*_3 \neq \emptyset$), then deployment of a static guard in $S^*_1 \cap S^*_2 \cap S^*_3$ is enough to cover the entire environment. Now consider the case in which two star regions intersect, $S^*_1 \cap S^*_2 \neq 0$, but third one ($S^*_3$) is disjoint (i.e. $S^*_1 \cap S^*_3 = 0,S^*_2 \cap S^*_3 = 0$). In such cases, a mobile guard on the shortest path between $S^*_1 \cap S^*_2$ and $S^*_3$ can track an evader for infinite time. 

When all three star regions are disjoint (i.e. $S^*_i \cap S^*_j = 0, i \neq j, i,j:1,2,3$), one mobile guard can track an evader. In this case, we partition the environment based on the motion of pursuer on the shortest path, which connects the furthest star regions. Figure~\ref{Septagon_model} shows a septagon and three disjoint star regions. The pursuer's trajectory contains two paths $H_1O_1$ and $O_1H_2$. In this case $v^*_p$ is given by the following expression:
\begin{eqnarray}
v^*_p = \mathop {\min }\limits_{r_1,r_2} {\max\{\frac{d_1\bar{v}_e}{r_1},\frac{d_2\bar{v}_e}{r_2}\}},
\end{eqnarray}
where $d_1$ and $d_2$ are lengths of $H_1O_1$ and $O_1H_2$, receptively. $r_1$ and $r_2$ are the radius of sector of dynamic zones around $O_1$ and $O_2$, respectively. Note that $r_1$ and $r_2$ are subject to the following constraints:
\begin{eqnarray}
r_1 \leq o_{12}, \qquad r_2 \leq o_{13}, \qquad r_1+r_2 \leq o_{23}.
\end{eqnarray}
Since $\frac{d_1\bar{v}_e}{r_1},\frac{d_2\bar{v}_e}{r_2}$ are decreasing function of $r_1,r_2$, $v^*_p$ occurs at the boundary of one of the constraints.

\end{itemize} 

In all other cases, we can construct a road map among the intersection of star regions and disjoint star regions. The lower bound for $\bar{v}_p$ is $\min \max \{\frac{d_i}{r_i}\bar{v}_e\}$.

According to the proposition~\ref{number of star regions}, octagons have at most three disjoint star regions. Hence, all the above cases hold for general octagonal environments.
\begin{figure}[!tbp]
  \centering
  \subfloat[]{\includegraphics[width=.25\textwidth]{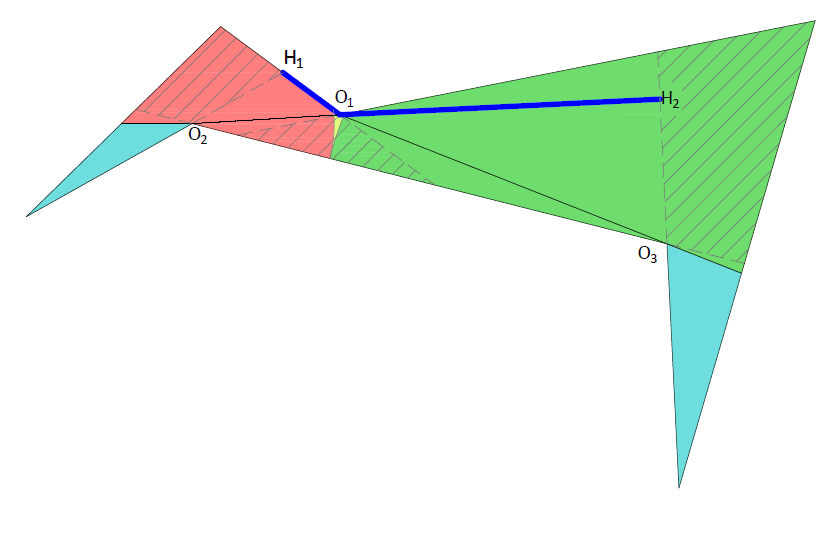}\label{Septagon_model}}
  \hfill
  \subfloat[]{\includegraphics[width=.22\textwidth]{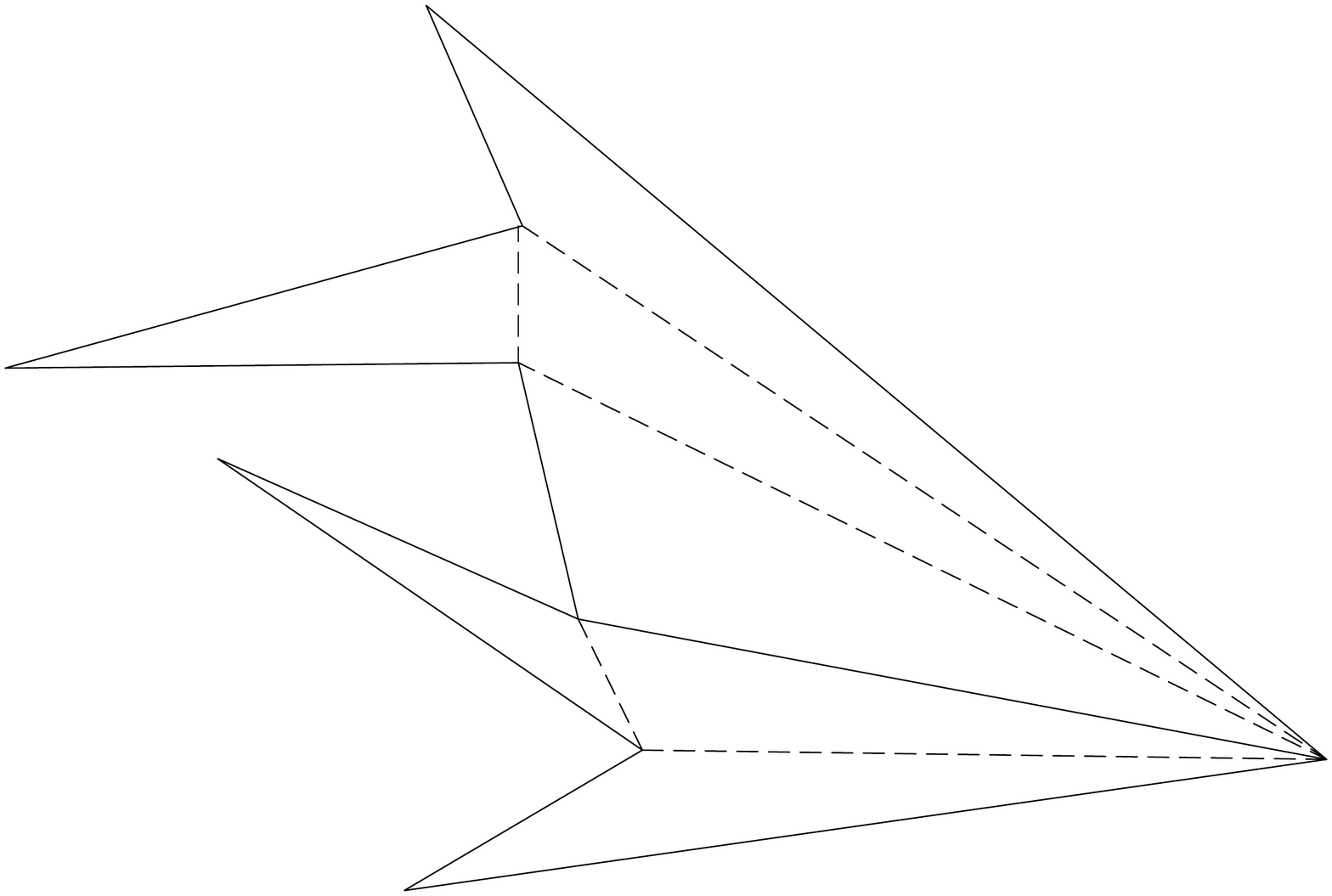}\label{Nonagon}}
  \caption{(a) Septagonal environment,(b) Partitioning 9-gon into hexagon and pentagon}
\end{figure}
\noindent

\subsubsection*{Deployment of a guard in a hexagon, septagon and octagon}
As mentioned earlier, only one guard is enough to track an evader within a hexagon, a septagon or an octagon. To deploy a guard, we need to identify all star regions in these polygons, and subsequently partition the polygon based on the number of star regions and maximum speed of evader. Deployment of the guard depends on the position of evader with respect to the partitions. Moreover, $v^*_p$ is computed based on $\bar{v}_e$ and geometry of the star regions. As the evader moves inside the polygon, the guard should adopt an appropriate strategy with respect to the evader's position in the partitions. 
\subsection{Tracking in a Nonagon}
In this subsection, we show that two guards, one static and one mobile guard, is enough to track an evader in a 9-gon. In every triangulation of a 9-gon, we can find a diagonal such that it partitions the environment into two parts, one of which is a pentagon and the other one is a hexagon. Deploying a static guard is enough to cover the pentagon, and a mobile guard can track the evader within the hexagon for infinite time. Hence, two guards are sufficient to guard a 9-gon.

Now we have shown that hexagons, septagons, octagons and nonagons can be guarded by static or mobile guards such that the number of required guards is strictly less than the number of required static guards to cover these environments. Therefore, according to proposition~\ref{prop mobile guards}, any general polygon can be guarded with group of $n_g$ static and mobile guards such that  $n_g < \lfloor\frac{n}{3}\rfloor$. The procedure of deploying guards is based on the minimal partitioning of polygon $P$ into hexagons, septagons, octagons and nonagons. Each partition is guarded by one or two guards based on the number of reflex vertices and the relative placement of star regions. Algorithm~\ref{algorithm-general-polygon} presents the procedure of deploying guards in a general polygon. Figure~\ref{algorithm-general-polygon-figure} shows minimal partitioning of a 20-gon and deployment of guards for infinite time tracking of an evader. Partitioning and motion strategy of mobile guards, based on the intruder's trajectory, are shown in our video submission.

	\begin{algorithm}
	\caption{Infinite time tracking of the intruder in a general polygon.}\label{algorithm-general-polygon}
	\begin{algorithmic}[1]
	\Procedure{Deploying guards }{$P,e_0$}\\
	\textbf{Input:} $P$ is a polygon, $e_0$ is the initial position of the intruder.
	\State Minimal partitioning $P$ to ${\cal{P}}=\{P_1,\dots,P_r\}$
	\State $v^*_p \gets {0}$
	\For  {\text{any}$P_i \in \cal{P}$}
	\If {$P_i$ is hexagon, septagon or octagon,} 
	\State Follow the deployment of a guard in a hexagon, septagon or octagon.

	\Else {}
	\State Follow the deployment of guards in a nonagon environment.
	\EndIf
	\State Compute ${v^*_p}_i$.
	\State $v^*_p \gets \max({v^*_p}_i,v^*_p)$
	\EndFor

	\EndProcedure
	\end{algorithmic}	
	\end{algorithm}

   \begin{figure}
   \centering
   \includegraphics[scale=.4]{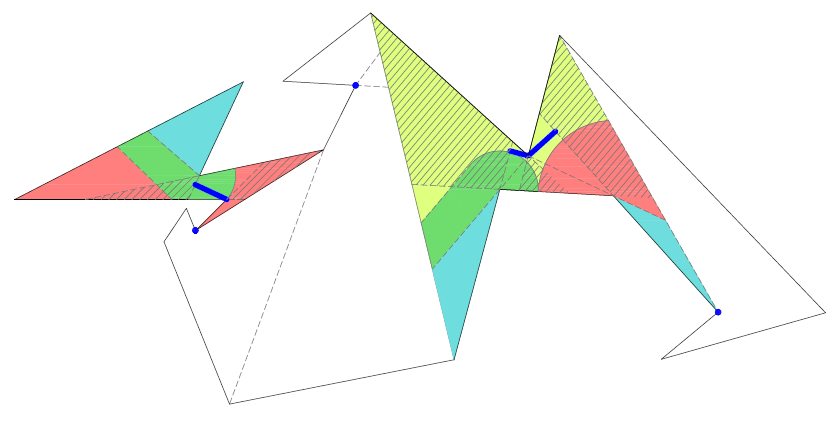}
   \caption{Deployment of guards in a 20-gon. Solid blue lines show the paths of mobile guards and blue dots are static guards.}
   \label{algorithm-general-polygon-figure}
   \end{figure}

\section{Target Tracking in Orthogonal Polygons} 
In this section, we present deployment strategies for guards in orthogonal polygons, which is an important subclass of polygons. An orthogonal polygon is one whose edges are all parallel to the axes of Cartesian coordinate. Therefore, all the angles of an orthogonal polygon are either $90^\circ$ or $270^ \circ$. We take all edges of $P$ to be horizontal and vertical without loss of generality.  We begin with the following lemmas.
\begin{lemma}
\cite{o1987art} Every orthogonal polygon $P$ is convexly quadrilateralizable.
\end{lemma}

Given the quadrilateralization of an orthogonal polygon, we can associate a graph similar to the dual graph of the triangulation of a polygon. Let $G$ be a graph such that every vertex corresponds to a quadrilateral in quadrilateralization $Q$ of $P$, and two vertices share an edge if their corresponding quadrilaterals share an edge.
\begin{lemma}
\cite{o1987art} For every quadrilateralization $Q$ of an orthogonal polygon $P$, the dual graph $G$ is a tree with each node of maximum degree 4.
\end{lemma}
\begin{lemma} \label{partitioning orthogonal polygon}
The quadrilateralization $Q$ of $P$ can be partitioned by diagonals of $P$ into smaller partitions, $P_i$'s, such that each $P_i$ contains $2,3,$ or $4$ quadrilaterals. Furthermore, there may be one remaining quadrilateral, denoted by $P^\prime$.
\end{lemma}
\begin{proof}
Let $Q$ be a quadrilateralization of $P$, and $G(v,e)$ be a dual graph of $Q$. Now choose edge $e_0 \in e$ which separates off a minimum number of nodes, that is at least $2$. Let $q \geq 2$ be this minimum, and $v_0 \in v$ be the end node of link $e_0$ in the separated part. The degree of nodes adjacent to $v_0$ should be less than 2, because if degree of those nodes is greater than 1, then a cut through $e_0$ is not a minimum cut. Therefore, the separated part has at most 4 nodes (i.e. 4 quadrilaterals). The remaining part of $G$ is a tree with maximum degree 4, and it can be separated off again. Based on the number of nodes of $G$, one node may remain at the end of the cutting process.
\end{proof}

The sufficient number of static guards to cover an orthogonal polygon, which is $\lfloor \frac{n}{4}\rfloor$, is based on assigning a guard to each pair of quadrilaterals. In order to track an evader inside the orthogonal environment, we assign a mobile guard to more than two quadrilaterals, and as a consequence we will reduce the number of static guards. We will proceed by partitioning $P$ into smaller partitions, such that each partition contains $q=2,3$ or $4$ quadrilaterals. We will show that for each case, only a mobile guard is enough to track an evader for infinite time within a partition which contains $q$ quadrilaterals. First, we show that assigning a mobile guard to each partition will result in fewer than $\lfloor \frac{n}{4}\rfloor$ mobile guards. 

\begin{lemma}
Orthogonal polygon $P$ can be guarded by less than $\lfloor \frac{n}{4}\rfloor$ mobile guards if each partition $P_i$ can be guarded by one mobile guard, and $P_i$'s included at least two partitions with $q=3$, or one partition with $q=4$.
\end{lemma}
\begin{proof}
According to lemma~\ref{partitioning orthogonal polygon}, any orthogonal polygon $P$ can be partitioned into parts which contain only 2,3, or 4 quadrilaterals. Let $n_2,n_3,n_4$ be the number of partitions with 2,3,4 quadrilaterals, respectively. $k^\prime \in \{0,1\}$ corresponds to the last quadrilateral $P^\prime$. We will show $n_2+n_3+n_4 +k^\prime < \lfloor \frac{n}{4}\rfloor = \lfloor \frac{r+1}{2}\rfloor$, where $r$ is the total number of partitions.
\begin{eqnarray}
r &=& 2n_2+3n_3+4n_4+k^\prime,\\
\frac{r+1}{2} &=& n_2+n_3+n_4+\frac{k^\prime+1}{2}+\frac{n_3+2n_4}{2},\\
\lfloor \frac{r+1}{2}\rfloor &=&  n_2+n_3+n_4 +k^\prime + \lfloor \frac{n_3+2n_4}{2}\rfloor.
\end{eqnarray}
For $n_3 \geq 2$ or $n_4 \geq 1$,$\lfloor \frac{n_3+2n_4}{2}\rfloor \geq 1$, lemma holds.   
\end{proof}
When all $P_i$'s contain only two quadrilaterals, that is the dual graph $G$ is a line graph, we can take each pair of $P_i$'s and assign one mobile guard to them instead of assigning two static guards to them and reduce the required number of guards.

In case I, we assume that diagonal $d$ separates off $P_1$, which contains 4 convex quadrilaterals. Without loss of generality, assume $d$ is a diagonal $AB$ in quadrilateral $ABCD$. Quadrilateral $ABCD$ corresponds to a node with degree $4$ in $G(v,e)$. Reflex vertices in a quadrilateral of degree 4 can have only one configuration, which is depicted in figure~\ref{orthogonal-polygon}(a). Since both edges adjacent to $d$, which belong to $P_1$, should have the same orientation (i.e. both should be horizontal or vertical), both end points of $d$ cannot be reflex vertices in polygon $P_1$ simultaneously. Therefore, polygon $P_1$ has three reflex vertices, two of which are $270^{\circ}(3\pi/2)$. Hence, two star regions should be quadrants and the connecting line between these reflex vertices is always inside one of the star regions. As a consequence, these star regions intersect, and the required maximum speed of guard ($v^*_p$) can be found, based on the geometry of $P_1$ and maximum speed of evader.
  \begin{figure}
   \centering
   \includegraphics[scale=.23]{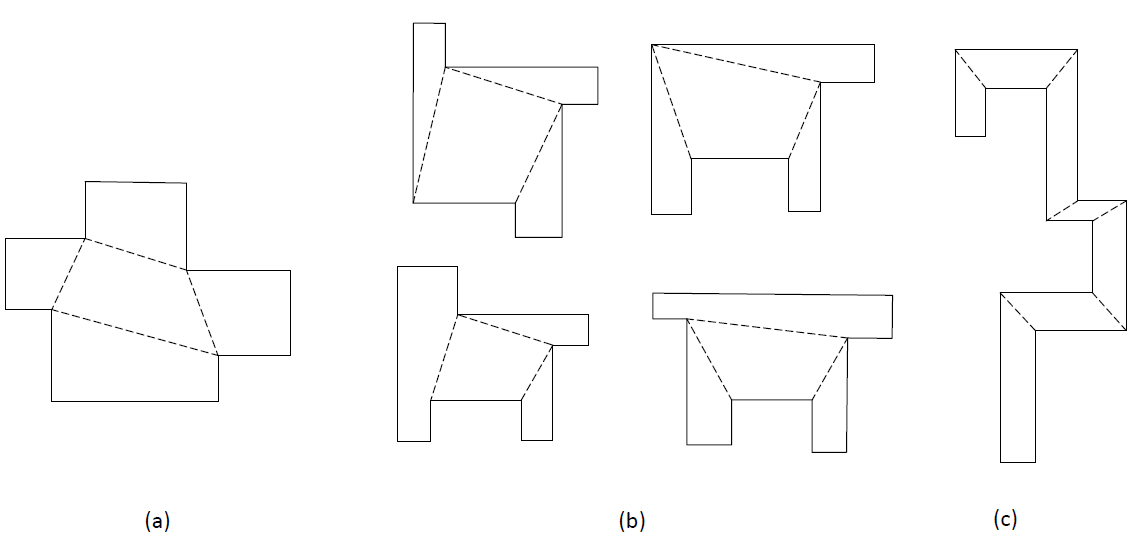}
   \caption{(a) The only configuration for a degree 4 quadrilateral. (b) The possible configurations for a degree 3 quadrilateral. (c) The  possible configuration for a degree 2 quadrilateral}
   \label{orthogonal-polygon}
   \end{figure}

In case II, we assume that partition $P_1$ contains $3$ quadrilaterals, i.e. $q=3$, and $d$ belongs to a quadrilateral which corresponds to a node with degree $3$ on graph $G$. A quadrilateral of degree three can have four configurations. These configurations are depicted in figure~\ref{orthogonal-polygon}(b). $P_1$ contains at most two reflex vertices in all possible configurations. Consequently, one mobile guard can track an evader within this part if the maximum speed of guard be greater than the required speed $v^*_p$. The necessary speed $v^*_p$ can be characterized based on the geometry of star regions and maximum speed of evader $\bar{v}_e$. 

When minimal partitioning $P$ results in only 2 quadrilaterals, which implies that dual graph $G$ is a line graph. We can assign a mobile guard to each three consecutive quadrilaterals. In this case, $P_1$ contains three quadrilaterals and each quadrilateral has a pair of horizontal or vertical edges. In this case, $P_1$ contains only two disjoint star regions. Figure~\ref{orthogonal-polygon}(c) shows such a condition.

\section{CONCLUSIONS}
In this work, we have proposed an algorithm to deploy a group of stationary and mobile guards in order to track an intruder inside a simply connected polygon. The environment is partitioned into smaller polygons, and a guard is assigned to each partition. Guards can be stationary or mobile depending on the geometry of the partition. The intruder is considered as a mobile agent with bounded speed. We provided a strategy to deploy stationary guards. For mobile guards, we presented an event-triggered strategy for their motion based on the intruder's trajectory in the dynamic zones. We have shown that the total number of guards is less than $\lfloor {\frac{n}{3}} \rfloor$, which is the sufficient number of point guards deployed for the coverage problem. 

An ongoing effort in our research is to investigate cooperative task allocation among guards. The proposed algorithm is based on assigning a guard to each partition. Let $G$ be a dual graph in which each node corresponds to a partition in $P$, and there an edge exists between the nodes if they share an edge in triangulation $T$. In the current algorithm, once the intruder enters the partition, the assigned guard tracks it independently. However, cooperation among guards which are adjacent in $G$ will reduce the load on each guard. Another direction of the future research is to construct a road map for the guards which satisfies the necessary conditions of tracking. The current road map for the mobile guards is the line segment connecting the star regions or their intersections. This road map provides an upper bound for the maximum speed of the guard. However, the problem of finding the optimal speed for the guard can be cast as a convex optimization problem for certain subclasses of polygons, for example, polygons in which star regions are disjoint, and non-adjacent reflex vertices on the dual graph are not mutually visible. 
Other future directions of research are to incorporate sensing constraints, non-holonomic constraints and extension to 3 dimensional environments.

\addtolength{\textheight}{-12cm}   
\bibliographystyle{IEEEtran}

\bibliography{ref}

\end{document}